\def\year{2021}\relax
\documentclass[letterpaper]{article} 
\usepackage{aaai21}  
\usepackage{times}  
\usepackage{helvet} 
\usepackage{courier}  
\usepackage[hyphens]{url}  
\usepackage{graphicx} 
\usepackage{natbib}  
\usepackage{caption} 
\usepackage{complexity}
\usepackage{amsmath,amsthm}
\newtheorem{theorem}{Theorem}

\usepackage[T1]{fontenc}
\usepackage[utf8]{inputenc}

\usepackage{amsmath}
\usepackage{amsfonts}
\usepackage{amssymb}
\usepackage{breqn}

\usepackage[ruled,vlined]{algorithm2e}

\usepackage{stmaryrd}
\usepackage{tikz}
\usetikzlibrary{arrows.meta}
\usetikzlibrary{calc}
\usepackage[caption=false,font=footnotesize]{subfig}
\usepackage{stfloats}
\usepackage{hyperref}
\usepackage{url}

\definecolor{tabblue}{HTML}{0e79b1}
\newcommand{\giacobbeno}{\cite{giacobbe2019many}}
\newcommand{\giacobbe}{\cite{giacobbe2019many} }
\newcommand{\zvonimir}{\cite{baranowski2020fixed} }
\newcommand{\factoringpaper}{\cite{cheng2018verification} }
\newcommand{\reals}{\mathbb{R}}
\newcommand{\inputdomain}{\mathcal{D}}
\newcommand{\sem}[1]{\llbracket #1 \rrbracket}
\newcommand{\semreal}[1]{\sem{#1}_\reals}
\newcommand{\semfloat}[1]{\sem{#1}_{\text{float32}}}
\newcommand{\semink}[2]{\sem{#1}_{\text{int-}#2}}

\begin{document}
	
	\title{Scalable Verification of Quantized Neural Networks (Technical Report)}
	\author{Thomas A. Henzinger, Mathias Lechner, \DJ{}or\dj{}e \v{Z}ikeli\'c}
	\affiliations{
		
		IST Austria\\
		Klosterneuburg, Austria\\
		\{tah,mlechner, dzikelic\}@ist.ac.at
		
	}

\maketitle

\begin{abstract}
Formal verification of neural networks is an active topic of research, and recent advances have significantly increased the size of the networks that verification tools can handle. However, most methods are designed for verification of an idealized model of the actual network which works over real arithmetic and ignores rounding imprecisions. This idealization is in stark contrast to network quantization, which is a technique that trades numerical precision for computational efficiency and is, therefore, often applied in practice. Neglecting rounding errors of such low-bit quantized neural networks has been shown to lead to wrong conclusions about the network's correctness. Thus, the desired approach for verifying quantized neural networks would be one that takes these rounding errors into account. 
In this paper, we show that verifying the bit-exact implementation of quantized neural networks with bit-vector specifications is PSPACE-hard, even though verifying idealized real-valued networks and satisfiability of bit-vector specifications alone are each in NP. Furthermore, we explore several practical heuristics toward closing the complexity gap between idealized and bit-exact verification. In particular, we propose three techniques for making SMT-based verification of quantized neural networks more scalable. Our experiments demonstrate that our proposed methods allow a speedup of up to three orders of magnitude over existing approaches.
\end{abstract}

\section{Introduction}
Deep neural networks for image classification typically consist of a large number of sequentially composed layers. Computing the output of such a network for a single input sample may require more than a billion floating-point operations \cite{tan2019efficientnet}.
Consequently, deploying a trained deep neural network imposes demanding requirements on the computational resources available at the computing device that runs the network.
Quantization of neural networks is a technique that reduces the computational cost of running a neural network by reducing the arithmetic precision of computations inside the network \cite{jacob2018quantization}.
As a result, quantization has been widely adapted in industry for deploying neural networks in a resource-friendly way. For instance, Tesla's Autopilot Hardware 3.0 is designed for running 8-bit quantized neural networks \cite{fsdchip}).

The verification problem for neural networks consists of checking validity of some input-output relation. More precisely, given two conditions over inputs and outputs of the network, the goal is to check if for every input sample which satisfies the input condition, the corresponding output of the neural network satisfies the output condition. Verification of neural networks has many important practical applications such as checking robustness to adversarial attacks~\cite{szegedy2013intriguing,tjeng2019evaluating}, proving safety in safety-critical applications~\cite{huang2017safety, lechner2021infinite, lechner2021stability} or output range analysis~\cite{dutta2019reachability}, to name a few. There are many efficient methods for verification of neural networks (e.g.~\cite{katz2017reluplex,tjeng2019evaluating,bunel2018unified}), however most of them ignore rounding errors in computations. The few approaches that can handle the semantics of rounding operations are overapproximation-based methods, i.e., incomplete verification \cite{singh2018fast,singh2019abstract}.
The imprecision introduced by quantization stands in stark contrast with the idealization made by verification methods for standard neural networks, which disregards rounding errors that appear due to the network's semantics.
Consequently, verification methods developed for standard networks are not sound for and cannot be applied to quantized neural networks. Indeed, recently it has been shown that specifications that hold for a floating-point representation of a network need not necessarily hold after quantizing the network~\cite{giacobbe2019many}.
As a result, specialized verification methods that take quantization into account need to be developed, due to more complex semantics of quantized neural networks.
Groundwork on such methods demonstrated that special encodings of networks in terms of Satisfiability Modulo Theories (SMT) \cite{smt2018handbook} with bit-vector \cite{giacobbe2019many} or fixed-point \cite{baranowski2020fixed} theories present a promising approach towards the verification of quantized networks.
However, the size of networks that these tools can handle and runtimes of these approaches do not match the efficiency of advanced verification methods developed for standard networks like Reluplex\cite{katz2017reluplex} and Neurify \cite{wang2018efficient}.

In this paper, we provide first evidence that the verification problem for quantized neural networks is harder compared to verification of their idealized counterparts, thus explaining the scalability-gap between existing methods for standard and quantized network verification. In particular, we show that verifying quantized neural networks with bit-vector specifications is \PSPACE-hard, despite the satisfiability problem of formulas in the given specification logic being in NP. As verification of neural networks without quantization is known to be \NP-complete~\cite{katz2017reluplex}, this implies that the verification of quantized neural networks is a harder problem.

We then address the scalability limitation of SMT-based methods for verification of quantized neural networks, and propose three techniques for their more efficient SMT encoding.
First, we introduce a technique for identifying those variables and constraints whose value can be determined in advance, thus decreasing the size of SMT-encodings of networks. Second, we show how to encode variables as bit-vectors of minimal necessary bit-width. This significantly reduces the size of bit-vector encoding of networks in~\cite{giacobbe2019many}. Third, we propose a redundancy elimination heuristic which exploits bit-level redundancies occurring in the semantics of the network.

Finally, we propose a new method for the analysis of the quantized network's reachable value range, which is based on abstract interpretation and assists our new techniques for SMT-encoding of quantized networks. We evaluate our approach on two well-studied adversarial robustness verification benchmarks. Our evaluation demonstrates that the combined effect of our techniques is a speed-up of over three orders of magnitude compared to the existing tools.

The rest of this work is organized as follows: First, we provide background and discuss related works on the verification of neural networks and quantized neural networks.
We then start with our contribution by showing that the verification problem for quantized neural networks with bit-vector specifications is PSPACE-hard.
In the following section, we propose several improvements to the existing SMT-encodings of quantized neural networks.
Finally, we present our experimental evaluation to assess the performance impacts of our techniques.

\section{Background and Related work}\label{sec:background}
A neural network is a function $f: \reals^n \rightarrow \reals^m$ that consists of several layers $f = l_1 \circ l_2\circ \cdots\circ l_k$ that are sequentially composed, with each layer parameterized by learned weight values.
Commonly found types of layers are linear 
\begin{equation}\label{eq:linear_layer}
    l(x) = Wx + b, W\in \reals^{n_o \times n_i}, b\in \reals^{n_o},
\end{equation}
ReLU $l(x) = \max\{x,0\}$, and convolutional layers \cite{lecun1998gradient}. 

In practice, the function $f$ is implemented by floating-point arithmetic instead of real-valued computations.
To distinguish a neural network from its approximation, we define an interpretation $\sem{f}$ as a map which assigns a new function to each network, i.e.
\begin{equation}
    \sem{ }: (\reals^n \rightarrow \reals^m) \rightarrow (\inputdomain \rightarrow \reals^m),
\end{equation}
where $\inputdomain \subset \reals^n$ is the admissible input domain.
For instance, we denote by $\sem{f}_\reals: f \mapsto f$
the idealized real-valued abstraction of a network $f$, whereas $\semfloat{f}$ denotes its floating-point implementation, i.e.~the realization of $f$ using 32-bit IEEE floating-point \cite{kahan1996ieee} instead of real arithmetic.
Evaluating $f$, even under floating-point interpretation, can be costly in terms of computations and memory resources.
In order to reduce these resource requirements, networks are usually quantized before being deployed to end devices \cite{jacob2018quantization}. 

Formally, quantization is an interpretation $\semink{f}{k}$ that evaluates a network $f$ which uses $k$-bit fixed-point arithmetic \cite{smith1997scientist}, e.g. 4 to 8 bits. Let $[\mathbb{Z}]_k=\{0,1\}^k$ denote the set of all bit-vectors of bit-width $k$. For each layer $l:[\mathbb{Z}]_k^{n_i}\rightarrow [\mathbb{Z}]_k^{n_0}$ in $\semink{f}{k}$, we define its semantics by defining $l(x_1,\dots,x_{n_i})=(y_1,\dots,y_{n_0})$ as follows:
\begin{align}
    x'_i &= \sum_{j=1}^{n_i} w_{ij} x_j + b_i,\label{eq:sum}\\
    x''_i &= \text{round}(x'_i,k_i) = \lfloor x'_i\cdot 2^{-k_i} \rfloor,\qquad \text{and}\label{eq:round} \\
    y_i &= \max\{0,\min\{2^{N_i}-1, x''_i \}\}\label{eq:relun},
\end{align}
Here, $w_{i,j}$ and $b_i$ for each $1\leq j\leq n_i$ and $1\leq i\leq n_0$ denote the learned weights and biases of $f$, and $k_i$ and $N_i$ denote the bit-shift and the cut-off value associated to each variable $y_i$, respectively. Eq.~\eqref{eq:sum} multiplies the inputs $x_j$ with the weight values $w_{ij}$ and adds the bias $b_i$, eq.~\eqref{eq:round} rounds the result to the nearest valid $k$-bit fixed-point value, and eq.~\eqref{eq:relun} is a non-linear ReLU-N activation function \footnote{Note that for quanitzed neural networks, the double-side bounded ReLU-N activation is preferred over the standard ReLU activation function \cite{jacob2018quantization}}.



An illustration of how the computations inside a network differ based on the used interpretation is shown in Fig. \ref{fig:precision}.

\begin{figure}[t]
    \centering
     \begin{tikzpicture}[
neuron/.style={circle,draw,fill=tabblue,inner sep=0.2cm},
neuronp/.style={circle,draw,fill=tabblue!50!black!20,inner sep=0.15cm}]
\node at (-3.2,0.6) [anchor=west] {\textbf{A) } Idealized real-valued network $\semreal{f}$};
\node (fake1) at (-1,-0.25) [anchor=east] {$0.94374\dots$};
\node (fake2) at (-1,-1.25) [anchor=east] {$1.382723\dots$};
\node (out) at (3,-0.75) [anchor=west] {$2.57799431\dots$};
\node (i1) at (0,-0.25) [neuron] {};
\node (i2) at (0,-1.25) [neuron] {};
\node (s1) at (2,-0.75) [neuronp] {+};
\draw[-Latex] (fake1) to (i1);
\draw[-Latex] (fake2) to (i2);
\draw[-Latex] (s1) to (out);
\draw[-Latex] (i1) to node[above] {1.75} (s1);
\draw[-Latex] (i2) to node[below] {0.67} (s1);
\begin{scope}[yshift=-3.0cm]
\node at (-3.2,0.6) [anchor=west] {\textbf{B)} Floating-point network $\semfloat{f}$};
\node (fake1) at (-1,-0.25) [anchor=east] {$0.94374$};
\node (fake2) at (-1,-1.25) [anchor=east] {$1.3827$};
\node (out) at (3,-0.75) [anchor=west,align=left] {$\semfloat{2.577954}$\\[0.2cm]\ $=2.5780$};
\node (i1) at (0,-0.25) [neuron] {};
\node (i2) at (0,-1.25) [neuron] {};
\node (s1) at (2,-0.75) [neuronp] {+};
\draw[-Latex] (fake1) to (i1);
\draw[-Latex] (fake2) to (i2);
\draw[-Latex] (s1) to (out);
\draw[-Latex] (i1) to node[above] {1.75} (s1);
\draw[-Latex] (i2) to node[below] {0.67} (s1);
\end{scope}
\begin{scope}[yshift=-6cm]
\node at (-3.2,0.6) [anchor=west] {\textbf{C)} Quantized (fixed-point) network $\semink{f}{8}$};
\node (fake1) at (-1,-0.25) [anchor=east] {$0.94$};
\node (fake2) at (-1,-1.25) [anchor=east] {$1.38$};
\node (out) at (3,-0.75) [anchor=west,align=left] {$\semink{2.5696}{8}$\\[0.2cm]\ $=2.57$};
\node (i1) at (0,-0.25) [neuron] {};
\node (i2) at (0,-1.25) [neuron] {};
\node (s1) at (2,-0.75) [neuronp] {+};
\draw[-Latex] (fake1) to (i1);
\draw[-Latex] (fake2) to (i2);
\draw[-Latex] (s1) to (out);
\draw[-Latex] (i1) to node[above] {1.75} (s1);
\draw[-Latex] (i2) to node[below] {0.67} (s1);
\end{scope}
\end{tikzpicture}

    \caption{Illustration of how different interpretations of the same network run with different numerical precision. \textbf{A)} $\semreal{f}$ assumes infinite precision. \textbf{B)} $\semfloat{f}$ rounds the mantissa according on the IEEE 754 standard. \textbf{C)} $\semink{f}{8}$ rounds to a fixed number of digits before and after the comma. (Note that this figure serves as a hypothetical example in decimal format, the actual computations run with the base-2 representation.)}
    \label{fig:precision}
    \vspace{-1.5em}
\end{figure}

\subsection{Verification of neural networks}
The verification problem for a neural network and its given interpretation consists of verifying some input-output relation. More formally, given a neural network $f$, its interpretation $\sem{f}$ and two predicates $\varphi$ and $\psi$ over the input domain $\inputdomain$ and output domain $\reals^m$ of $\sem{f}$, we want to check validity of the following formula (i.e.~whether it holds for each $x\in\inputdomain$)
\begin{equation}\label{eq:verifiproblem}
    \varphi(x)\land \sem{f}(x)=y \Longrightarrow \psi(y).
\end{equation}
We refer to the formula in eq.~\eqref{eq:verifiproblem} as the formal specification that needs to be proved. In order to formally verify a neural network, it is insufficient to just specify the network without also providing a particular interpretation. A property that holds with respect to one interpretation need not necessarily remain true if we consider a different interpretation. For example, robustness of the real-valued abstraction does not imply robustness of the floating-point implementation of a network \cite{giacobbe2019many,jia2020exploiting}.

Ideally, we would like to verify neural networks under the exact semantics that are used for running networks on the end device, i.e., $\semfloat{f}$ most of the time.
However, as verification methods for IEEE floating-point arithmetic are extremely inefficient, research has focused on verifying the idealized real-valued abstraction $\semreal{f}$ of $f$.
In particular, efficient methods have been developed for a popular type or networks that only consist of linear and ReLU operations (Figure \ref{fig:relun} a) \cite{katz2017reluplex,ehlers2017formal,tjeng2019evaluating,bunel2018unified}.
The piecewise linearity of such ReLU networks allows the use of Linear Programming (LP) techniques, which make the verification methods more efficient. 
The underlying verification problem of ReLU networks with linear inequality specifications was shown to be NP-complete in the number of ReLU operations \cite{katz2017reluplex}, however advanced tools scale beyond toy networks.

Although these methods can handle networks of large size, they are building on the assumption that 
\begin{equation}
    \semfloat{f} \approx \semreal{f},
\end{equation}
i.e.~that the rounding errors introduced by the IEEE floating-point arithmetic of both the network and the verification algorithm can be neglected. 
It has been recently shown that this need not always be true.
For example, Jia and Rinard \cite{jia2020exploiting} crafted adversarial counterexamples to the floating-point implementation of a neural network whose idealized interpretation was verified to be robust against such attacks, by exploiting subtle numerical differences between $\semfloat{f}$ and $\semreal{f}$.

\subsection{Verification of quantized neural networks}
The low numerical precision of few-bit fixed-point arithmetic implies that $\semink{f}{k} \neq \semreal{f}$. Indeed, \giacobbe constructed a prototypical network that either satisfies or violates a formal specification, depending on the numerical precision used to evaluate the network. Moreover, they observed such discrepancy in networks found in practice. Thus, no formal guarantee on $\semink{f}{k}$ can be obtained by verifying $\semreal{f}$ or $\semfloat{f}$.
In order to verify fixed-point implementations of (i.e.~quantized) neural networks, new approaches are required.



\begin{figure}[t]
    \centering
    	\begin{tikzpicture}[]
    	\node at (-0.8,2.0) {\large\textbf{a)}};
	\draw[->] (-0.4,0) to (2.2,0);
	\draw[->] (0,-0.2) to (0,1.8);
	\draw[line width=0.08cm,color=tabblue] (-0.6,0) -- (0,0) -- (1.8,1.8);
	\node at (2.0,0.3) {$x$};.5,0) -- (0,0) -- (1.5,1.5);
	\node at (-0.3,1.9) {$y$};
	\node at (0.5,-0.8) {$y = \big\llbracket \text{ReLU}(x) \rrbracket_{\reals}$};
	\begin{scope}[xshift={4cm}]
	    	\node at (-0.8,2.0) {\large\textbf{b)}};
	\draw[->] (-0.4,0) to (2.2,0);
	\draw[->] (0,-0.2) to (0,1.8);
	\draw[line width=0.08cm,color=tabblue] (-0.6,0) -- (0,0) -- ++(0,0.2) -- ++(0.2,0) -- ++(0,0.2) -- ++(0.2,0)-- ++(0,0.2) -- ++(0.2,0)-- ++(0,0.2) -- ++(0.2,0)-- ++(0,0.2) -- ++(0.2,0)-- ++(0,0.2) -- ++(0.2,0) -- ++(0,0.2) -- ++(0.2,0) -- ++(0.6,0);
	\node at (2.0,0.3) {$x$};.5,0) -- (0,0) -- (1.5,1.5);
	\node at (-0.3,1.9) {$y$};
	\node at (0.5,-0.8) {$y = \semink{\text{ReLU-N}(x)}{k}$};
	\end{scope}
	\end{tikzpicture}
    \caption{Illustration of \textbf{a)} the ReLU activation function under real-valued semantics, and \textbf{b)} ReLU-N activation under fixed-point semantics (right).}
    \label{fig:relun}
    \vspace{-1.5em}
\end{figure}
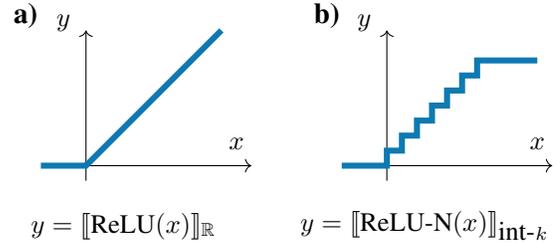

Fig.~\ref{fig:relun} depicts the ReLU activation function for idealized real-valued ReLU networks and for quantized ReLU networks, respectively. The activation function under fixed-point semantics consists of an exponential number of piecewise constant intervals thus making the LP-based techniques, which otherwise work well for real-valued networks, extremely inefficient. So the approaches developed for idealized real-valued ReLU networks cannot be efficiently applied to quantized networks.
Existing verification methods for quantized neural networks are based on bit-exact Boolean Satisfiability (SAT) and SMT encodings.
For 1-bit networks, i.e., binarized neural networks, Narodytska et al. \cite{narodytska2018verifying} and \factoringpaper proposed to encode the network semantics and the formal specification into an SAT formula, which is then checked by an off-the-shelf SAT solver.
While their approach could handle networks of decent size, the use of SAT-solving is limited to binarized networks, which are not very common in practice.

\giacobbe proposed to verify many-bit quantized neural network by encoding their semantics and specifications into quantifier-free bit-vector SMT (QF\_BV) formulas. 
The authors showed that, by reordering linear summations inside the network, such monolithic bit-vector SMT encodings could scale to the verification of small but interestingly sized networks.

\zvonimir introduced an SMT theory for fixed-point arithmetic and showed that the semantics of quantized neural networks could be encoded in this theory very naturally.
However, as the authors only proposed prototype solvers for reference purposes, the size of the verified networks was limited.

\subsection{Limitations of neural network verification}
The existing techniques for verification of idealized real-valued abstractions of neural networks have significantly increased the size of networks that can be verified~\cite{ehlers2017formal,katz2017reluplex,bunel2018unified,tjeng2019evaluating}. However, scalability remains the key challenge hindering formal verification of neural networks in practice.
For instance, even the largest networks verified by the existing methods~\cite{ruan2018reachability} are tiny compared to the network architectures used for object detection and image classification \cite{he2016deep}.

Regarding the verification of quantized neural networks, no advanced techniques aiming at performance improvements have been studied so far.
In this paper, we address the scalability of quantized neural network verification methods that rely on SMT-solving.

\section{Hardness of Verification of Quantized Neural Networks}\label{sec:hardness}

The size of quantized neural networks that existing verification methods can handle is significantly smaller compared to the real arithmetic networks that can be verified by the state-of-the-art tools like~\cite{katz2017reluplex,tjeng2019evaluating,bunel2018unified}. Thus, a natural question is whether this gap in scalability is only because existing methods for quantized neural networks are less efficient, or if the verification problem for quantized neural networks is computationally harder.

In this section, we study the computational complexity of the verification problem for quantized neural networks. For idealized real arithmetic interpretation of neural networks, it was shown in~\cite{katz2017reluplex} that, if predicates on inputs and outputs are given as conjunctions of linear inequalities, then the problem is \NP-complete. The fact that the problem is \NP-hard is established by reduction from $3$-\SAT, and the same argument can be used to show that the verification problem for quantized neural networks is also \NP-hard. In this work, we argue that the verification problem for quantized neural networks with bit-vector specifications is in fact \PSPACE-hard, and thus harder then verifying real arithmetic neural networks. Moreover, we show that this holds even for the special case when there are no constraints on the inputs of the network, i.e.~when the predicate on inputs is assumed to be a tautology. The verification problem for a quantized neural network $f$ that we consider consists of checking validity of a given input-output relation formula
\begin{equation*}
    \semink{f}{k}(x)=y \Longrightarrow \psi(y).
\end{equation*}
Here, $\semink{f}{k}$ is the $k$-bit fixed point arithmetic interpretation of $f$, and $\psi$ is a predicate in some specification logic over the outputs of $\semink{f}{k}$. Equivalently, we may also check satisfiability of the dual formula
\begin{equation}\label{eq:satisfiability}
    \semink{f}{k}(x)=y \land \neg\psi(y).
\end{equation}


In order to study complexity of the verification problem, we also need to specify the specification logic to which formula $\psi$ belongs. In this work, we study hardness with respect to the fragment $\textsc{QF\_BV2}_{bw}$ of the fixed-size bit-vector logic $\textsc{QF\_BV2}$~\cite{kovasznai2016complexity}. The fragment $\textsc{QF\_BV2}_{bw}$ allows bit-wise logical operations (such as bit-wise conjunction, disjunction and negation) and the equality operator. The index $2$ in $\textsc{QF\_BV2}_{bw}$ is used to denote that the constants and bit-widths are given in binary representation. It was shown in~\cite{kovasznai2016complexity} that the satisfiability problem for formulas in $\textsc{QF\_BV2}_{bw}$ is \NP-complete.

Even though $\textsc{QF\_BV2}_{bw}$ itself allows only bit-vector operations and not linear integer arithmetic, we show that by introducing dummy output variables in $\semink{f}{k}$ we may still encode formal specifications on outputs that are boolean combinations of linear inequalities over network's outputs. Thus, this specification logic is sufficiently expressive to encode formal specifications most often seen in practice. Let $y_1,\dots,y_m$ denote output variables of $\semink{f}{k}$. In order to encode an inequality of the form $a_1y_1+\dots+a_my_m+b\geq 0$ into the output specification, we do the following:
\begin{itemize}
    \item Introduce an additional output neuron $\tilde{y}$ and a directed edge from each output neuron $y_i$ to $\tilde{y}$. Let $a_i$ be the weight of an edge from $y_i$ to $\tilde{y}$, $b$ be the bias term of $\tilde{y}$, $k-1$ be the bit-shift value of $\tilde{y}$, and $N=k$ be the number of bits defining the cut-off value of $\tilde{y}$. Then
    \[
    \tilde{y}=\text{ReLU-N}(\text{round}(2^{-(k-1)}(a_1y_1+\dots+a_my_m+b))).
    \]
    Thus, as we work with bit-vectors of bit-width $k$, $\tilde{y}$ is just the sign bit of $a_1y_1+\dots+a_sy_s+b$ preceded by zeros.
    \item As $a_1y_1+\dots+a_sy_s+b\geq 0$ holds if and only if the sign bit of $a_1y_1+\dots+a_sy_s+b$ is $0$, in order to encode the inequality into the output specification it suffices to encode that $\tilde{y}=\mathbf{0}$, which is a formula expressible in $\textsc{QF\_BV2}_{bw}$.
\end{itemize}
By doing this for each linear inequality in the specification and since the logical operations are allowed by $\textsc{QF\_BV2}_{bw}$, it follows that we may use $\textsc{QF\_BV2}_{bw}$ to encode boolean combinations of linear inequalities over outputs as formal specifications that are to be verified.

Our main result in this section is that, if $\psi$ in eq.~\eqref{eq:satisfiability} is assumed to be a formula in $\textsc{QF\_BV2}_{bw}$, then the verification problem for quantized neural networks is \PSPACE-hard. Since checking satisfiability of $\psi$ can be done in non-deterministic polynomial time, this means that the additional hardness really comes from the quantized neural networks.

\begin{theorem}[Complexity of verification of QNNs]\label{thm:complexity}
If the predicate on outputs is assumed to be a formula in $\textsc{QF\_BV2}_{bw}$, the verification problem for quantized neural networks is \PSPACE-hard.
\end{theorem}

\begin{proof}[Proof sketch]
Here we summarize the key ideas of our proof. For the complete proof, see the appendix.

To prove \PSPACE-hardness, we exhibit a reduction from \textsc{TQBF} which is known to be \PSPACE-complete~\cite{AroraB2009}. \textsc{TQBF} is the problem of deciding whether a quantified boolean formula (QBF) of the form $Q_1x_1.\,Q_2x_2.\,\dots\,Q_nx_n.\, \phi(x_1,x_2,\dots,x_n)$ is true, where each $Q_i\in\{\exists,\forall\}$ and $\phi$ is a quantifier-free formula in propositional logic over the variables $x_1,\dots,x_n$. A QBF formula is true if it admits a truth table for each existentially quantified variable $x_i$, where the truth table for $x_i$ specifies a value in $\{0,1\}$ for each valuation of those universally quantified variables $x_j$ on which $x_i$ depends (i.e.~$x_j$ with $j<i$). Thus, the size of each truth table is at most $2^k$, where $k$ is the total number of universally quantified variables in the formula.

In our reduction, given an instance of the \textsc{TQBF} problem $Q_1x_1.\,Q_2x_2.\,\dots\,Q_nx_n.\, \phi(x_1,x_2,\dots,x_n)$ we map it to the corresponding verification problem as follows. The interpretation $\semink{f}{k}$ of the neural network $f$ consists of $n+1$ disjoint gadgets $f_1,\dots,f_n,g$, each having a single input and a single output neuron of bit-width $2^k$. Note that bit-widths are given in binary representation, thus this is still polynomial in the size of the problem. We use these gadgets to encode all possible inputs to the QBF formula, whereas the postcondition in the verification problem encodes the quantifier-free formula itself. For a universally quantified variable $x_i$, the output of $f_i$ is always a constant vector encoding the values of $x_i$ in each of the $2^k$ valuations of universally quantified variables (for a fixed ordering of the valuations). For existentially quantified $x_i$, we use $f_i$ and its input neuron to encode $2^k$ possible choices for the value of $x_i$, one for each valuation of universally quantified variables, and thus to encode the truth table for $x_i$.
Finally, the gadget $g$ is used to return a constant bit-vector $\mathbf{1}$ of bit-width $2^k$ on any possible input. 
The predicate $\psi$ on the outputs is then defined as
\[
\psi:=(\phi_{bw}(y_1,\dots,y_n) = \mathbf{1}),
\]
where $\phi_{bw}$ is the quantifier-free formula in $\textsc{QF\_BV2}_{bw}$ identical to $\phi$, with only difference being that the inputs of $\phi_{bw}$ are bit-vectors of bit-width $2^k$ instead of boolean variables, and logical operations are also defined over bit-vectors (again, since bit-widths are encoded in binary representation, this is of polynomial size). The output of $\phi_{bw}$ is thus tested if it equals $1$ for each valuation of universally quantified variables and the corresponding values of existentially quantified variables defined by the truth tables. Our construction ensures that any satisfying input for the neural networks induces satisfying truth tables for the TQBF instance and vice-versa, which completes the reduction.
\end{proof}

Theorem~\ref{thm:complexity} is to our best knowledge the first theoretical result which indicates that the verification problem for quantized neural networks is harder than verifying their idealized real arithmetic counterparts. It sheds some light on the scalability gap of existing SMT-based methods for their verification, and shows that this gap is not solely due to practical inefficiency of existing methods for quantized neural networks, but also due to the fact that the problem is computationally harder. While Theorem~\ref{thm:complexity} gives a lower bound on the hardness of verifying quantized neural networks, it is easy to see that an upper bound on the complexity of this problem is $\NEXP$ since the inputs to the verification problem are of size that is exponential in the size of the problem. Closing the gap and identifying tight complexity bounds is an interesting direction of future work.

Note though that the specification logic $\textsc{QF\_BV2}_{bw}$ used to encode predicates over outputs is strictly more expressive than what we need to express boolean combinations of linear integer inequalities, which is the most common form of formal specifications seen in practice. This is because $\textsc{QF\_BV2}_{bw}$ also allows logical operations over bit vectors, and not just over single bits. Nevertheless, our result presents the first step towards understanding computational hardness of the quantized neural network verification problem.

\section{Improvements to bit-vector SMT-encodings}\label{sec:advanced_smt}

In this section, we study efficient SMT-encodings of quantized neural networks that would improve scalability of verification methods for them. In particular, we propose three simplifications to the monolithic SMT encoding of eq.~\eqref{eq:sum}, \eqref{eq:round}, and \eqref{eq:relun} introduced in \giacobbeno, which encodes quantized neural networks and formal specifications as formulas in the QF\_BV2 logic :
I) Remove dead branches of the If-Then-Else encoding of the activation function in eq.~\eqref{eq:relun}, i.e., branches that are guaranteed to never be taken;
II) Allocate only the minimal number of bits for each bit-vector variable in the formula; and III) Eliminate sub-expressions from the summation in eq.~\eqref{eq:sum}.
To obtain the information needed by the techniques I and II we further propose an abstract interpretation framework for quantized neural networks. 
\vspace{-0.1em}


\subsection{Abstract interpretation analysis}

Abstract interpretation \cite{cousot1977abstract} is a technique for constructing over-approximations to the behavior of a system. Initially developed for software verification, the method has recently been adapted to robustness verification of neural networks and is used to over-approximate the output range of variables in the network. Instead of considering all possible subsets of real numbers, it only considers an abstract domain which consists of subsets of suitable form (e.g.~intervals, boxes or polyhedra). This allows modeling each operation in the network in terms of operations over the elements of the abstract domain, thus over-approximating the semantics of the network. While it leads to some impreision, abstract interpretation allows more efficient output range analysis for variables. Due to its over-approximating nature, it remains sound for verifying neural networks.


Interval \cite{wang2018formal,tjeng2019evaluating}, zonotope \cite{mirman2018differentiable,singh2018fast}, and convex polytope \cite{katz2017reluplex,ehlers2017formal,bunel2018unified,wang2018efficient} abstractions have emerged in literature as efficient and yet precise choices for the abstract domains of real-valued neural networks. The obtained abstract domains have been used for output range analysis \cite{wang2018formal}, as well as removing decision points from the search process of complete verification algorithms \cite{tjeng2019evaluating,katz2017reluplex}.
One important difference between standard and quantized networks is the use of double-sided bounded activation functions in quantized neural networks, i.e., ReLU-N instead of ReLU \cite{jacob2018quantization}. 
This additional non-linear transition, on one hand, renders linear abstractions less effective, while on the other hand it provides hard upper bounds to each neuron, which bounds the over-approximation error.
Consequently, we adopt interval abstractions (IA) on the quantized interpretation of a network to obtain reachability sets for each neuron in the network. As discussed in \cite{tjeng2019evaluating}, using a tighter abstract interpretation poses a tradeoff between verification and pre-processing complexity. 
\vspace{-0.1em}

\subsection{Dead branch removal}
Suppose that through our abstract interpretation we obtained an interval $[lb,ub]$ for the input $x$ of a ReLU-N operation $y=\text{ReLU-N}(x)$. Then, we can substitute the formulation of the ReLU-N by
\begin{equation*}
 \begin{cases}
0, &\text{if } ub \leq 0\\
2^N-1, &\text{if } lb \geq 2^N-1\\
x, &\text{if } ub \geq 0 \text{ and }  lb \leq 2^N-1\\
\max\{0,x\}, &\text{if } 0 < ub \leq 2^N-1.\\
\min\{2^N-1,x\}, &\text{if } 0 \leq lb < 2^N-1.\\
 \max\{0,\min\{2^N-1,x\}\},& \text{otherwise,}
\end{cases}
\end{equation*}
which reduces the number of decision points in the SMT formula.


\subsection{Minimum bit allocation}
A $k$-bit quantized neural network represents each neuron and weight variable by a $k$-bit integer.
However, when computing the values of certain types of layers, such as the linear layer in eq. \eqref{eq:linear_layer}, a wider register is necessary.
The binary multiplication of a $k$-bit weight and a $k$-bit neuron value results in a number that is represented by $2k$-bits. Furthermore, summing up $n$ such $2k$-bit integer requires 
\begin{equation}\label{eq:naivebits}
    b_{\text{naive}} = 2k+\log_2(n)+1
\end{equation} bits to be safely represented without resulting in an overflow. 

Thus, linear combinations are in practice usually computed on 32-bit integer registers. Application of fixed-point rounding and the activation function then reduces the neuron values back to a $k$-bit representation~\cite{jacob2018quantization}.

QF\_BV2 reasons over fixed-size bit-vectors, i.e.~the bit width of each variable must be fixed in the formula regardless of the variable's value.
\giacobbe showed that the number of bits used for all weight and neuron variables in the formal affects the runtime of the SMT-solver significantly. 
For example, omitting the least significant bit of each variable cuts the runtime on average by half.
However, the SMT encoding of \giacobbe allocates $b_{\text{naive}}$ bits according to eq.~\eqref{eq:naivebits} for each accumulation variable of a linear layer.

Our approach uses the interval $[lb,ub]$ obtained for each variable by abstract interpretation to compute the minimal number of bits necessary to express any value in the interval.
As the signed bit-vector variables are represented in the two's complement format, we can compute the bit width $b$ of variable $x$ with computed interval $[lb,ub]$ by
\begin{equation}
    b_{\text{minimal}} = 1+\log_2(\max\{|lb|,|ub|\}+1).
\end{equation}
Trivially, one can show that $b_{\text{minimal}} < b_{\text{naive}}$, as $|ub| \leq 2^{2k}n$ and $|lb| \leq 2^{2k}n$.

\subsection{Redundant multiplication elimination}
Another difference between quantized and standard neural networks is the rounding of the weight values to the nearest representable value of the employed fixed-point format.
Consequently, there is a considerable chance that two connections outgoing from the same source neuron will have the same weight value.
For instance, assuming an 8-bit network and a uniform weight distribution, the chance of two connections having the same weight value is around $0.4\%$ compared to the much lower $4\cdot 10^{-8}\%$ of the same scenario happening in a floating-point network.

Moreover, many weight values express some subtle form of redundancy on a bit-level.
For instance, both multiplication by 2 and multiplication by 6 contain a shift operations by 1 digit in their binary representation.
Thus, computations
\begin{align}
    y_1 = 3\cdot x_1 & & y_2 = 6\cdot x_1
\end{align}
can be rewritten as
\begin{align}
    y_1 = 3\cdot x_1 & & y_2 = y_1<<1,
\end{align}
where $<<$ is a binary shift to the left by 1 digit.
As a result, a multiplication by 6 is replaced by a much simpler shift operation. Based on this motivation, we propose a redundancy elimination heuristic to remove redundant and partially redundant multiplications from the SMT formula.
Our heuristic first orders all outgoing weights of a neuron in ascending order and then sequentially applies a rule-matching for each weight value.
The rules try to find a simpler way to compute the multiplication of the weight and the neuron value by using already performed multiplications.
The algorithm and the rules in full are provided in the appendix.

Note that a similar idea was introduced by \factoringpaper in the form of a neuron factoring algorithm for the encoding of binarized (1-bit) neural networks into SAT formulas. However, the heuristic of \factoringpaper removes redundant additions, whereas we consider bit-level redundancies in multiplications. 
For many-bit quantization, the probability of two neurons sharing more than one incoming weight is negligible, thus making such neuron factoring proposed in \factoringpaper less effective.


\section{Experimental Evaluation}\label{sec:experiments}
We create an experimental setup to evaluate how much the proposed techniques affect the runtime and efficiency of the SMT-solver.
Our reference baseline is the approach of \giacobbeno, which consists of a monolithic and "balanced" bit-vector formulation for the Boolector SMT-solver.
We implement our techniques on top of this baseline. We limited our evaluation to Boolector, as other SMT-solvers supporting bit-vector theories, such as Z3 \cite{de2008z3}, CVC4 \cite{barrett2011cvc4}, and Yices \cite{dutertre2014yices}, performed much worse in the evaluation of \giacobbeno. 

Our evaluation comprises of two benchmarks. Our first evaluation considers the adversarial robustness verification of image classifier trained on the MNIST dataset \cite{lecun1998gradient}. 
In particular, we check the $l_{\infty}$ robustness of networks against adversarial attacks \cite{szegedy2013intriguing}. Other norms, such as $l_1$ and $l_2$, can be expressed in bit-vector SMT constraints as well, although with potentially negative effects on the solver runtime.
In the second evaluation, we repeat the experiment on the slightly more complex Fashion-MNIST dataset \cite{xiao2017online} .

All experiments are run on a 14-core Intel W-2175 CPU with 64GB of memory. We used the boolector \cite{NiemetzPreinerBiere-JSAT15} with the SAT-solvers Lingeling\cite{lingeling} (only baseline) and CaDiCal \cite{cadical} (baseline + our improvements) as SAT-backend.

Adversarial robustness specification can be expressed as
\begin{equation}\label{eq:robust}
|x-x_i|_\infty \leq \varepsilon \land y=\semink{f}{k}(x) \implies y = y_i,
\end{equation}
where $(x_i,y_i)$ is a human labeled test sample and $\varepsilon$ is a fixed attack radius.
As shown in eq.~\eqref{eq:robust}, the space of possible attacks increases with $\varepsilon$. 
Consequently, we evaluate with different attack radii $\varepsilon$ and study the runtimes individually.   
In particular, for MNIST we check the first 100 test samples with an attack radius of $\varepsilon=1$, the next 100 test samples with $\varepsilon=2$, and the next 200 test samples with $\varepsilon=3$ and $\varepsilon=4$ respectively.
For our Fashion-MNIST evaluation, we reduce the number of samples to 50 per attack radius value for $\varepsilon>2$ due to time and compute limitations.

The network studied in our benchmark consists of four fully-connected layers (784,64,32,10), resulting in 52,650 parameters in total. It was trained using a quantization-aware training scheme with a 6-bit quantization.

The results for the MNIST evaluation in terms of solved instances and median solver runtime are shown in Table \ref{tab:mnist_solved} and Table \ref{tab:mnist_runtime} respectively.
Table \ref{tab:fashion_solved} and Table \ref{tab:mnist_runtime} show the results for the Fashion-MNIST benchmark.

\begin{table}
    \centering
    \begin{tabular}{c|ccc}
Attack   & Baseline & Baseline & Ours  \\
radius &  (+ Lingeling) & (+ CaDiCal) &  \\\hline
$\varepsilon=1$ &  63 (63.6\%) & 92 (92.9\%) & \textbf{99 (100.0\%)} \\
$\varepsilon=2$ &  0 (0.0\%) & 20 (20.2\%) & \textbf{94 (94.9\%)} \\
$\varepsilon=3$ & 0 (0.0\%) & 2 (2.1\%) & \textbf{71 (74.0\%)}\\
$\varepsilon=4$ &  0 (0.0\%) & 1 (1.0\%) & \textbf{54 (55.7\%)}\\
    \end{tabular}
\caption{Number of solved instances of adversarial robustness verification on the MNIST dataset. Absolute numbers and in percentages of checked instances in parenthesis.}
\label{tab:mnist_solved}
\end{table}

\begin{table}
	\centering
	\begin{tabular}{c|ccccccc}
		Dataset   & Baseline & Baseline & Ours \\
		 &  (+ Lingeling) & (+ CaDiCal) &  \\\hline
		MNIST & 8803 \textbar 8789 & 2798 \textbar 3931 & \textbf{5} \textbar \textbf{90} \\\hline
		Fashion-MNIST & 6927 \textbar 6927 & 3105 \textbar 3474 & \textbf{4} \textbar \textbf{49} \\\hline
	\end{tabular}
	\caption{Median \textbar mean runtime of adversarial robustness verification process per sample. The reported values only account for non-timed-out samples.}
	\label{tab:mnist_runtime}
	\vspace{-1.05em}
\end{table}

\begin{table}
    \centering
    \begin{tabular}{c|ccc}
Attack   & Baseline & Baseline & Ours \\
radius &  (+ Lingeling) & (+ CaDiCal) &  \\\hline
$\varepsilon=1$ & 2 (2.3\%) & 44 (50.6\%) & \textbf{76 (87.4\%)} \\
$\varepsilon=2$ & 0 (0.0\%) & 7 (7.8\%) & \textbf{73 (81.1\%)} \\
$\varepsilon=3$ &  0 (0.0\%) & 1 (2.3\%) & \textbf{27 (62.8\%)} \\
$\varepsilon=4$ & 0 (0.0\%) & 0 (0.0\%) & \textbf{18 (40.9\%)}\\
    \end{tabular}
\caption{Number of solved instances of adversarial robustness verification on the Fashion-MNIST dataset. Absolute numbers and in percentages of checked instances in parenthesis. Best method in bold.}
\label{tab:fashion_solved}
\end{table}

\begin{table}
	\centering
	\begin{tabular}{l|rr}
		Method   & Total solved & Cumulative  \\
		   & instances & runtime \\\hline
		No redundancy eliminiation &  316 (80.8\%) & 7.7 h \\
		No minimum bitwidth & 315 (80.6\%) & 5.1 h \\
		No ReLU simplify & 88 (22.5\%) & 83.2 h \\
		No Abstract interpretation & 107 (27.4\%) & 126.0 h \\\hline
		All enabled & \textbf{318 (81.3\%)} & 7.9 h \\
	\end{tabular}
	\caption{Results of our ablation analysis on the MNIST dataset. The cumulative runtime only accounts for non-timed-out samples.}
	\label{tab:ablation}
	\vspace{-1.05em}
\end{table}

\subsection{Ablation analysis}
We perform an ablation analysis where we re-run our robustness evaluation with one of our proposed techniques disabled. The objective of our ablation analysis is to understand how the individual techniques affect the observed efficiency gains.
Due to time and computational limitations we focus our ablation experiments to MNIST exclusively.

The results in Table \ref{tab:ablation} show the highest number of solved instances were achieved when all our techniques were enabled. Nonetheless, Table \ref{tab:ablation} demonstrate these gains are not equally distributed across the three techniques. 
In particular, the ReLU simplification has a much higher contribution for explaining the gains compared to the redundancy elimination and minimum bitwidth methods.  
The limited benefits observed for these two techniques may be explain by the inner workings of the Boolector SMT-solver.

The Boolector SMT-solver \cite{NiemetzPreinerBiere-JSAT15} is based on a portfolio approach which sequentially applies several different heuristics to find a satisfying assignment of the input formula \cite{wintersteiger2009concurrent}. In particular, Boolector starts with fast but incomplete local search heuristics and falls back to slower but complete bit-blasting \cite{smt2018handbook} in case the incomplete search is unsuccessful \cite{niemetz2019boolector}. 
Although our redundancy elimination and minimum bitwidth techniques simplify the bit-blasted representation of the encoding, it introduces additional dependencies between different bit-vector variables. As a result, we believe these extra dependencies make the local search heuristics of Boolector less effective and thus enabling only limited performance improvements.

\section{Conclusion}
We show that the problem of verifying quantized neural networks with bit-vector specifications on the inputs and outputs of the network is PSPACE-hard.
We tackle this challenging problem by proposing three techniques to make the SMT-based verification of quantized networks more efficient.
Our experiments show that our method outperforms existing tools by several orders of magnitude on adversarial robustness verification instances.
Future work is necessary to explore quantized neural network verification's complexity with respect to different specification logics.
On the practical side, our methods point to limitations of monolithic SMT-encodings for quantized neural network verification and suggest that future improvements may be obtained by integrating the encoding and the solver steps more tightly.

\section{ Acknowledgments}
This research was supported in part by the Austrian Science Fund (FWF) under grant Z211-N23 (Wittgenstein Award), ERC CoG 863818 (FoRM-SMArt), and the European Union’s Horizon 2020 research and innovation programme under the Marie Skłodowska-Curie Grant Agreement No. 665385.

\bibliography{references}

\clearpage
\appendix

\begin{center}
	\Large\textbf{Appendix}
\end{center}

\section{Redundancy elimination algorithm}

The algorithm aiming to remove bit-level redundancies is shown in Algorithm \ref{algo:elim}. The rules for matching a weight value to the set of existing computations $V$ of a layer is Table \ref{tab:rules}.

\begin{algorithm}
    \SetAlgoLined
    \KwIn{Outgoing weights $W=\{w_i|i=1,\dots n\}$ of neuron $x$, with $n$ neurons in the next layers}
    \KwOut{Outgoing values $w_i\cdot x$ of neuron $x$}
    Sort $W$ in ascending order by absolute value\;
    $V \leftarrow \{\}$, $Y \leftarrow \{\}$\;
    \ForEach{$w_i \in W$}{
        Find rule for $w_i$ according to Table 1 given $V$\;
        \uIf{rule found}{
            $Y \leftarrow Y \cup \{rule(w_i,V)\}$\;
        }\Else{
            $y \leftarrow w_i\cdot x$\;
            $V \leftarrow V \cup \{w_i\}$, $Y \leftarrow Y \cup \{y\}$\;
        }
    }
    \textbf{return } $Y$\;
    \caption{Multiplication redundancy elimination}
    \label{algo:elim}
\end{algorithm}

\begin{table}
    \centering
    \begin{tabular}{l|l}
        Condition & Action \\\hline
        $w_i=0$ & $y_i=0$\\
        $w_i=1$ & $y_i=x$\\
        $\exists w_j: w_i=w_j$ & $y_i=y_j$\\
        $\exists w_j: w_i=-w_j$ & $y_i=-y_j$\\
        $\exists w_j: w_i=w_j\cdot 2^k$ & $y_i=y_j<<k$\\
    \end{tabular}
    \caption{Rules used for the multiplication redundancy elimination heuristic}
    \label{tab:rules}
    \vspace{-1.2em}
\end{table}

\section{Proof of Theorem~1}

In order to prove that the verification problem for quantized neural networks is \PSPACE-hard, we exhibit a reduction from \textsc{TQBF} which is known to be \PSPACE-complete~\cite{AroraB2009} to the QNN verification problem. \textsc{TQBF} is the problem of deciding whether a quantified boolean formula (QBF) in propositional logic of the form $Q_1x_1.\,Q_2x_2.\,\dots\,Q_nx_n.\, \phi(x_1,x_2,\dots,x_n)$ is true, where each $Q_i\in\{\exists,\forall\}$ and $\phi$ is a quantifier-free formula in propositional logic over the variables $x_1,\dots,x_n$.

\medskip\noindent {\em TQBF. }Given $\Phi = Q_1x_1.\,Q_2x_2.\,\dots\,Q_nx_n.\, \phi(x_1,x_2,\dots,x_n)$ a QBF formula, for each variable $x_i$ let $u(i)$ be the number of universally quantified variables $x_j$ with $j<i$. Then, $\Phi$ is true if it admits a truth table for each existentially quantified variable $x_i$, where the truth table for $x_i$ specifies a value in $\{0,1\}$ for each valuation of $u(i)$ universally quantified variables that $x_i$ depends on. Hence, the size of the truth table for $x_i$ is $2^{u(i)}$. In particular, if $k$ is the total number of universally quantified variables, then to show that $\Phi$ is true it suffices to find $n-k$ truth tables with each of size at most $2^k$.

\medskip\noindent {\em Ordering of variable valuations.} Let $U$ denote the ordered set of all universally quantified variables in the QBF formula, where variables are ordered according to their indices. A {\em valuation} of $U$ is an assignment in $\{0,1\}^k$ of each variable in $U$. As a truth table for each existentially quantified variable $x_i$ is defined with respect to all variable valuations of universally quantified variables on which $x_i$ depends, it will be convenient to fix an ordering $\sqsubseteq$ of all $2^k$ valuations of $U$. For two valuations $(y_1,\dots,y_k)$ and $(y_1',\dots,y_k')$ in $\{0,1\}^k$, we say that  $(y_1,\dots,y_k) \sqsubseteq (y_1',\dots,y_k')$ if either they are equal or there exists an index $1\leq i\leq k$ such that $y_i<y_i'$ and $y_j=y_j'$ for $j>i$. Equivalently, $(y_1,\dots,y_k) \sqsubseteq (y_1',\dots,y_k')$ if and only if
\[
\sum_{i=1}^k y_i\cdot 2^i \leq \sum_{i=1}^k y'_i\cdot 2^i
\]
Thus this is the lexicographic ordering on "reflected" valuations, i.e.~the largest index having the highest priority. For brevity, we will still refer to it as the {\em lexicographic ordering}. For an existentially quantified variable $x_i$, its truth table can therefore also be defined by specifying a value in $\{0,1\}$ for each of the first $2^{u(i)}$ valuations of $U$ in the lexicographic ordering, since these are the orderings in which we consider all possible valuations of the first $u(i)$ universally quantified variables in $U$ with setting the remaining universally quantified variables to equal $0$. 

\medskip\noindent {\em Reduction.} We now proceed to the construction of our reduction. Given $\Phi = Q_1x_1.\,Q_2x_2.\,\dots\,Q_nx_n.\, \phi(x_1,x_2,\dots,x_n)$ a QBF formula, we need to construct
\begin{enumerate}
    \item a quantized neural network $f^{\Phi}$, and
    \item a predicate $\psi^{\Phi}$ over the outputs of the neural network,
\end{enumerate}
each of polynomial size in the size of $\Phi$, such that $\Phi$ is true if and only if the neural networks admits inputs that satisfy the verification problem.
\begin{enumerate}
    \item {\em Construction of the quantized neural network.} We construct $f^{\Phi}$ to consist of $n+1$ gadgets $f^{\Phi}_1,\dots,f^{\Phi}_n,g^{\Phi}$. Each gadget can be viewed as a sub-neural network, and all nodes in gadgets are of the bit-width $2^k$ each. Since bit-widths in quantized neural networks are given in binary representation, specifying bit-width is still of polynomial size. Each gadget $f^{\Phi}_i$ is associated to the variable $x_i$ in the QBF formula. The purpose of each gadget is to produce the output of the following form:
    \begin{itemize}
        \item If $x_i$ is universally quantified in $\Phi$, then $f^{\Phi}_i$ will return the single output neuron whose value is the constant bit-vector $c_i\in \{0,1\}^{2^k}$ whenever the predicate $\psi^{\Phi}$ is satisfied. For each $1\leq j\leq 2^k$, the component $c_i[j]$ will be equal to $1$ if and only if the value of $x_i$ in the $j$-th valuation of $U$ (w.r.t.~the lexicographic ordering) is equal to $1$. Thus, $c_i$ will encode values of $x_i$ in each valuation of $U$ when ordered lexicographically.
        \item If $x_i$ is existentially quantified in $\Phi$, then $f^{\Phi}_i$ will return a single output neuron which will encode a truth table for $x_i$. Recall, $x_i$ depends only on the first $u(i)$ universally quantified variables in $\Phi$, thus its truth table is of size $2^{u(i)}$. As the values of $x_i$ should remain invariant if the values remaining universally quantified variables are changed, we will encode the truth table for $x_i$ by first extracting the first $2^{u(i)}$ bits from the input neuron to encode the truth table itself, and then copying this block of bits $2^{k-u(i)}$ times in order to obtain an output bit-vector of bit-width $2^k$.
        \item The gadget $g^{\Phi}$ will return the constant bit-vector $\mathbf{1}$ consisting of all $1$'s (thus written in bold) whenever the predicate $\psi^{\Phi}$ is satisfied. Note, the constant bit-vector whose each bit is 1 is exponential in $k$ and thus exponential in the size of the TQBF problem.
        Hence, as we will later need to encode $\mathbf{1}$ into the predicate $\psi^{\Phi}$ over outputs of the quantized neural network, we cannot do it directly but use the quantized neural network to construct $\mathbf{1}$.
    \end{itemize}
    We now describe the architecture of the gadgets that can perform the tasks described above:
    \begin{itemize}
        \item The gadget $g^{\Phi}$ consists only of the input and the output layer. The input layer consists of the single neuron $x_g$, and three output neurons $y_g$, $y_g'$ and $y_g''$. Each edge between the layers has weight $1$, bias $0$ and the cut-off value $2^k$. The bit-shifts of the edges from $x_g$ to $y_g$, $y_g'$ and $y_g''$ are $0$, $2^k-1$ and $1$, respectively. The gadget $g^{\Phi}$ is accompanied by the predicate $\psi^{\Phi}_g$ which is satisfied if and only if $x_g=y_g=\mathbf{1}$. Formally, we define $\psi^{\Phi}_g$ via
        \[ \psi^{\Phi}_g := ((y_g' = 1) \land (\neg (y_g'') \lor y_g)).  \]
        To prove this, suppose first that $\psi^{\Phi}_g$ is true so we need to show that $x_g=y_g=\mathbf{1}$. Clearly, by our definition of the gadget we have $x_g=y_g$. Furthermore, $y_g' = 1$ implies that the first bit of $x_g=y_g$ is equal to $1$. Finally, observe that
        \begin{equation*}
        \begin{split}
            &(\neg (y_g'') \lor y_g) \equiv (y_g'' \rightarrow y_g) \\
            &\equiv (0\rightarrow y_g[1]) \land \bigwedge_{j=1}^{2^k-1} (y_g[j] \rightarrow y_g[j+1]),
        \end{split}
        \end{equation*}
        where the last inequality follows from the fact that $y_g''$ is obtained by shifting $y_g$ by $1$ bit. Here we use $y_g[j]$ to denote the $j$-th bit in $y_g$ for each $1\leq j\leq 2^k$, with $y_g[1]$ denoting the most significant bit in $y_g$. Then, since we already showed that the first bit of $x_g=y_g$ is equal to $1$, a simple induction on $j$ shows that all bits of $x_g=y_g$ are equal to $1$, i.e.~$x_g=y_g=\mathbf{1}$.
        
        Conversely, suppose that $x_g=y_g=\mathbf{1}$ so we need to show that $\psi^{\Phi}_g$ is true. The formula $\neg (y_g'') \lor y_g$ is trivially true as $y_g=\mathbf{1}$ and $y_g' = 1$ follows since $y_g'$ is obtained by shifting $y_g=\mathbf{1}$ by $2^k-1$ bits. This concludes the proof that $\psi^{\Phi}_g$ is true if and only if $x_g=y_g=\mathbf{1}$.
    
        
        

        \item For $f^{\Phi}_i$ corresponding to universally quantified variable $x_i$, let $B_i$ be the block of bits starting with $2^{u(i)}$ zeros followed by $2^{u(i)}$ ones. The bit-vector $c_i$ should then consist of $2^{k-u(i)-1}$ repetitions of the block $B_i$. The gadget $f^{\Phi}_i$ will thus consist of two sequentially composed parts. The first part $f^{\Phi}_{i,1}$ takes any bit-vector of bit-width $2^k$ as an input, and outputs a bit-vector of the same bit-width which starts with the block $B_i$ followed by zeros. The second part $f^{\Phi}_{i,2}$ takes the output of the first part as an input, and outputs $c_i$.
        
        $f^{\Phi}_{i,1}$ consists of $3$ layers: the input layer $L_0$ with the single input neuron, layer $L_1$ with two neurons, and output layer $L_2$ with a single output neuron. The input neuron in $L_0$ is set to coincide with the output neuron of $g^{\Phi}$ and thus equals $\mathbf{1}$. The cut-off value of each neuron in $f^{\Phi}_{i,1}$ is $N=2^k$. The weights of edges from the input neuron in $L_0$ to neurons in $L_1$ are set to $w_{01}'=w_{01}''=1$, the biases $b_1'=b_1''=0$ and the bit-shifts $F_{01}'=2^{u(i)}$ and $F_{01}''=2^{2u(i)}$. Hence, the output values of two neurons in $L_1$ will be the bit-vectors of bit-width $2^k$ that start with $2^{u(i)}$ (resp.~$2^{2u(i)}$) zeros, followed by ones. Finally, the weights of edges from neurons in $L_1$ to the output neuron in $L_2$ are set to $w_{12}'=1$ and $w_{12}''=-1$, the bias $b_2=0$ and the bit-shifts $F_{12}'=F_{12}''=0$. The output value of the neuron in $L_2$ will thus be a bit-vector of bit-width $2^k$ which starts with the block of bits $B_i$ and followed by zeros, as desired.
        
        $f^{\Phi}_{i,2}$ consists of $2(k-u(i)-1)+1$ layers, where the input layer coincides with the output layer of $f^{\Phi}_{i,1}$. Then for each $1\leq j\leq k-u(i)-1$, the $2j$-th layer consists of two neurons and the $(2j+1)$-st layer consists of a single neuron. The cut-off value of each neuron is $N=2^k$. The weights of each edge in $f^{\Phi}_{i,2}$ is $1$ and the bias of each neuron is $0$, thus we only need to specify the bit-shifts. For two edges from the neuron in the $(2j-1)$-st layer to neurons in the $(2j)$-th layer we set bit-shifts to be $F_{2j-1,2j}'=0$ and $F_{2j-1,2j}''=u(i)+j$, respectively. For two edges from neurons in the $2j$-th layer to the neuron in the $(2j+1)$-st layer both bit-shits are set to $0$. Given that the input $f^{\Phi}_{i,2}$ is a bit-vector of bit-width $2^k$ which starts with the block $B_i$ of length $2^{2u(i)}$ followed by zeros, by simple induction one can show that the output of the neuron in the $(2j+1)$-st layer is a bit-vector of bit-width $2^k$ which starts with $2^{j}$ copies of $B_i$ which are followed by zeros. Hence, the value of the output neuron of $f^{\Phi}_{i,2}$ will be $c_i$, as desired.
        
        \item For $f^{\Phi}_i$ corresponding to existentially quantified variable $x_i$, the neural network $f^{\Phi}_i$ will also consist of two sequentially composed parts. The first part $f^{\Phi}_{i,1}$ takes any bit-vector of bit-width $2^k$ as an input, and outputs a bit-vector of the same bit-width which starts with the same $2^{u(i)}$ bits as the input bit-vector but which are then followed by zeros. The second part $f^{\Phi}_{i,2}$ takes the output of the first part as an input, and outputs a bit-vector obtained by copying $2^{k-u(i)}$ times the block of the first $2^{u(i)}$ bits.
        
        $f^{\Phi}_{i,1}$ consists only of the input and the output layer. The input layer consists of two neurons, one of which coincides with the input neuron $x_g$ of the gadget $g^{\Phi}$. We denote the other input neuron by $z_i$. The output layer consists of two neurons $h_i$ and $h_g$. The weights of edges from $z_i$ to $h_i$ and from $x_g$ to $h_g$ are set to $1$, and weights of edges from $z_i$ to $h_g$ and from $x_g$ to $h_i$ are set to $0$. Biases and cut-off values of all edges between the layers are set to $0$ and $2^k$, respectively. All bit-shifts are also set to $0$, with the exception of the edge from $x_g$ to $h_g$ whose bit-shift is set to $2^{u(i)}$. These choices ensure that $h_i = z_i$ and that $h_g$ is a bit-vector that starts with $2^{u(i)}$ $0$-bits followed by $2^k-2^{u(i)}$ $1$-bits. The gadget $f^{\Phi}_{i,1}$ is accompanied by the predicate $\psi^{\Phi}_i$ defined via
        \[ \psi^{\Phi}_i := (h_i = (h_i \land \neg(h_g)). \]
        This choice of $\psi^{\Phi}_i$ together with the design of $f^{\Phi}_{i,1}$ enforce that $\psi^{\Phi}_i\land \psi^{\Phi}_g$ is satisfied if and only if $x_i=z_i$ and the last $2^k-2^{u(i)}$ bits of $x_i=z_i$ are equal to $0$.
        
        
        
        Since the goal of the second part is to just copy $2^{k-u(i)}$ times the block of the first $2^{u(i)}$ bits of the output $h_i$ of $f^{\Phi}_{i,1}$, the second part $f^{\Phi}_{i,2}$ is constructed analogously as in the case of neural networks corresponding to universally quantified variables above.
    \end{itemize}
    Recall, constant bit-vectors, bit-widths of bit-vectors as well as the number of bits used for rounding (i.e.~bit-shifts) are encoded in binary representation. Thus, each of the values used in the construction of gadgets $f^{\Phi}_i$ and $g^{\Phi}$ is encoded using at most $k$ bits, and is polynomial in the size of $\Phi$. On the other hand, from our construction one can check that each gadget consists of at most $2k+4$ neurons. Therefore, as there are $n+1$ gadgets the size of all networks combined is $O(k\cdot (2k+4)\cdot (n+1)) = O(n^3)$.
    
    \item {\em Construction of the output predicate $\psi^{\Phi}$.} Denote by $y_1,\dots,y_n,y_g$ the outputs of $f^{\Phi}_1,\dots,f^{\Phi}_n,g^{\Phi}$, respectively. We define $\psi^{\Phi}$ as
    \begin{equation}\label{eq:predicatepost}
        \psi^{\Phi} := (\phi_{bw}(y_1,\dots,y_n) =y_g) \land \psi^{\Phi}_{\text{auxiliary}},
    \end{equation}
    where $\phi_{bw}$ is the quantifier-free formula in $\textsc{QF\_BV2}_{bw}$ identical to $\phi$, with only difference being that the inputs of $\phi_{bw}$ are bit-vectors of bit-width $2^k$ instead of boolean variables and logical operations are also defined over bit-vectors. The formula $\psi^{\Phi}_{\text{auxiliary}}$ collects the auxiliary logical predicates that were introduced by our construction of each gadget above, i.e.
    \[ \psi^{\Phi}_{\text{auxiliary}} := \bigwedge_{i=1}^n\psi^{\Phi}_i\land \psi^{\Phi}_g. \]
    As $y_g=\mathbf{1}$ if and only if $\Psi^{\Phi}_g$ is satisfied, the formula $\psi^{\Phi}$ is true if and only if the equality $\phi_{bw}(y_1,\dots,y_n) =y_g$ holds for each component. Intuitively, $\psi^{\Phi}$ performs bit-wise evaluation of the formula $\phi$ on each component of bit-vector inputs, and then checks if each output is equal to $1$. The size of $\psi^{\Phi}$ is thus $O(|\phi_{bw}|) = O(|\phi|\cdot k + 3\cdot n + 3)=O(|\phi|\cdot n)$, where the additional factor $k$ comes from the fact that inputs of $\phi_{bw}$ are bit-vectors of bit-width $2^k$, and bit-widths are encoded in binary representation.
    
    Note that the above expression for $\psi^{\Phi}$ differs from that in the sketch proof of Theorem~\ref{thm:complexity}, where we omitted the formula $\psi^{\Phi}_{\text{auxiliary}}$ to simplify the presentation.
\end{enumerate}
Hence, the size of the instance of the quantized neural network verification problem to which we reduced $\Phi$ is $O(n^3+n\cdot |\phi|)$, which is polynomial in the size of $\Phi$.

\medskip\noindent {\em Correctness of reduction.} It remains to prove correctness of our reduction, i.e. that $\Phi$ is true if and only if the corresponding quantized neural network verification problem is satisfiable.

Suppose first that $\Phi$ is true, i.e.~that for each existentially quantified variable $x_i$ in $\Phi$ there exists a truth table $\mathbf{t}_i$ of size $2^{u(i)}$, such that any valuation of universally quantified variables $U$ together with the corresponding values of existentially quantified variables defined by truth tables form a satisfying assignment for the quantifier-free formula $\phi$ in $\Phi$. Consider the following set of inputs $z_1,\dots,z_n,x_g$ to gadgets $f^{\Phi}_1,\dots,f^{\Phi}_n,g^{\Phi}$:
\begin{itemize}
    \item If $x_i$ is universally quantified, then $z_i=\mathbf{0}$.
    \item If $x_i$ is existentially quantified, consider $\mathbf{t}_i$ as a bit-vector of bit-width $2^{u(i)}$ with elements ordered in such a way that corresponding valuations of universally quantified variables on which $x_i$ depends in $\Phi$ are ordered lexicographically. Then $z_i$ starts with a block of bits identical to $\mathbf{t}_i$, followed by zeros.
    \item $z_g=0$.
\end{itemize}
From our construction of neural networks and the predicate $\psi^{\Phi}$ we know that:
\begin{itemize}
    \item $g^{\Phi}(z_g)=\mathbf{1}$.
    \item If $x_i$ is universally quantified, then $f^{\Phi}_i(z_i)$ is equal to the bit-vector $c_i$ whose $j$-th component is equal to $1$ if and only if the value of $x_i$ in the $j$-th valuation of $U$ in the lexicographic ordering is equal to $1$, where $1\leq j\leq 2^k$.
    \item If $x_i$ is existentially quantified, then $f^{\Phi}_i(z_i)$ is the bit-vector obtained by copying the block $\mathbf{t}_i$ $2^{k-u(i)}$ times. Thus, the $j$-th component of $f^{\Phi}_i(z_i)$ is equal to the value in the truth table $\mathbf{t}_i$ corresponding to the $j$-th valuation of $U$ in the lexicographic ordering, where $1\leq j\leq 2^k$.
\end{itemize}
Finally, as $\psi^{\Phi}$ is obtained by considering a bit-vector version of formula $\phi$ and then checking if each component of the output is equal to $1$, it follows that the output of $\psi^{\Phi}$ on inputs $f^{\Phi}_1(z_1),\dots,f^{\Phi}_n(z_n)$ is equal to $1$, thus showing that the quantized neural network verification problem is satisfiable.

Conversely, suppose that $z_1,\dots,z_n,z_g$ is a set of satisfying inputs to the quantized neural network verification problems. Then for each existentially quantified variable $x_i$, we construct a truth table $\mathbf{t}_{i}$ as follows. Again, consider $\mathbf{t}_{i}$ as a bit-vector of bit-width $2^{u(i)}$, where elements are ordered in such a way that the corresponding valuations of universally quantified variables on which $x_i$ depends are ordered lexicographically. Then we set $\mathbf{t}_i$ to be equal to the block of first $2^{u(i)}$ bits in $z_i$. From our construction of the quantized neural network and $\psi^{\Phi}$, and the fact that $z_1,\dots,z_n,z_g$ is a satisfying inputs to the quantized neural network verification problem, it follows that for any valuation of $U$ the corresponding values of existentially quantified variables defined by these turth tables yield a satisfying assignment for $\phi$. Hence, the QBF formula $\Phi$ is true, as desired.

\end{document}